\documentclass[10pt,journal,cspaper,peerreview,compsoc]{IEEEtran}
\ifCLASSOPTIONcompsoc
  \usepackage[nocompress]{cite}
\else
  \usepackage{cite}
\fi
\usepackage{graphicx}
\graphicspath{{Images/}}
\usepackage{array}
\usepackage[cmex10]{amsmath}
\usepackage{amssymb}
\usepackage{amsthm}
\usepackage{algorithm}
\usepackage{algorithmic}
\usepackage{url}
\ifCLASSOPTIONcompsoc
  \usepackage[font=normalsize,labelfont=sf,textfont=sf]{subfig}
\else
  \usepackage[font=footnotesize]{subfig}
\fi
\usepackage{fixltx2e}

\newcommand{\bfX}{\mathbf{X}}
\newcommand{\bfx}{\mathbf{x}}
\newcommand{\bfy}{\mathbf{y}}
\newcommand{\bfk}{\mathbf{k}}
\newcommand{\bfh}{\mathbf{h}}
\newcommand{\bfb}{\mathbf{b}}
\newcommand{\bfZ}{\mathbf{Z}}
\newcommand{\bfI}{\mathbf{I}}
\newcommand{\bfU}{\mathbf{U}}
\newcommand{\bfH}{\mathbf{H}}
\newcommand{\bfm}{\mathbf{m}}
\newcommand{\bfv}{\mathbf{v}}
\newcommand{\bfw}{\mathbf{w}}
\newtheorem{theorem}{Theorem}
\newtheorem{lemma}{Lemma}

\begin{document}
%
\title{A Linear Approximation to the $\chi^2$ Kernel with Geometric Convergence}
%
%
%
%

\author{Fuxin~Li,~\IEEEmembership{Member,~IEEE,}
        Guy~Lebanon,~\IEEEmembership{Member,~IEEE,}
        and~Cristian~Sminchisescu,~\IEEEmembership{Member,~IEEE}
\IEEEcompsocitemizethanks{\IEEEcompsocthanksitem Fuxin Li is with the School of Interactive
Computing, Georgia Institute of Technology, Atlanta,
GA, 30332.\protect\\
E-mail: fli@cc.gatech.edu
\IEEEcompsocthanksitem Guy Lebanon is with the Department of Computational Science and Engineering,
Georgia Institute of Technology, Atlanta, GA, 30332.
\IEEEcompsocthanksitem Cristian Sminchisescu is with the Center for Mathematical Sciences,
Lund University, Lund, Sweden.}
\thanks{}}

%
%

\markboth{IEEE Transactions on Pattern Analysis and Machine Intelligence,~Vol.~XX, No.~X, January~2013}%
{On the $\chi^2$ and $\exp-\chi^2$ kernels}
%


\IEEEcompsoctitleabstractindextext{%
\begin{abstract}
We propose a new analytical approximation to the $\chi^2$ kernel that converges geometrically.
The analytical approximation is derived with elementary methods and adapts to the input distribution for optimal convergence rate.
Experiments show the new approximation leads to improved performance in image
classification and semantic segmentation tasks using a random Fourier feature
approximation of the $\exp-\chi^2$ kernel. Besides, out-of-core principal component
analysis (PCA) methods are introduced to reduce the dimensionality of the
approximation and achieve better performance at the expense of only an additional constant factor to the time complexity. Moreover, when
PCA is performed jointly on the training and unlabeled testing data, further performance improvements can be obtained.
Experiments conducted on the PASCAL VOC 2010 segmentation and
the ImageNet ILSVRC 2010 datasets show statistically significant improvements over
alternative approximation methods.
\end{abstract}

}

\maketitle


%
\IEEEpeerreviewmaketitle

\section{Introduction}
Histograms are important tools for
constructing visual object descriptors.
Many visual recognition approaches utilize similarity comparisons between histogram descriptors extracted
from training and testing images.
Widely used approaches such as $k$-nearest neighbors and support vector machines compare the testing
descriptor with multiple training descriptors, and make predictions
by a weighted sum of these comparison scores.

An important metric to compare histograms is the exponential-$\chi^2$ kernel (referred
to as $\exp-\chi^2$ in the rest of the paper),
derived from the classic Pearson $\chi^2$ test and utilized in many state-of-the-art
object recognition studies~\cite{Zhang07,Vedaldi2009,Marszalek2009context,Gonfaus2010} with excellent performance.
 However, in the current big data era,
 training sets often contains millions to billions of examples. Training and testing
via hundreds of thousands of comparisons using a nonlinear metric is often very time-consuming.

There are two main approaches to approximate the $\exp-\chi^2$ to facilitate fast linear
time training and testing. One approach is to devise
a transformation so that the $\chi^2$ function can be represented as an inner product between two vectors. On top of this
transformation, the random Fourier (RF) features methodology~\cite{Rahimi2007} is used
to approximate a Gaussian kernel. The full $\exp-\chi^2$ kernel can be approximated by inner
products on the vector after these two transformations~\cite{Sreekanth2010}.
A different approach
is the Nystr{\"o}m method~\cite{Williams2001}, which directly takes a subset of training examples, apply
the comparison metric between an example and this subset and use the output as the feature vector (sometimes
followed by principal component analysis (PCA)).

In this paper, we pursue the RF research line. We are interested in RF because it has
the potential of representing more complicated functions than the Nystr{\"o}m approach, which
is confined to summations of kernel comparisons and hard to approximate functions
not of that type. Besides, RF provides a fixed basis set regardless of input data, which could be valuable in online settings where a large training set is not available for sampling. However, RF has not been able to outperform Nystr{\"o}m so far, especially on image data with the $\exp-\chi^2$ approximation.

We believe that a partial reason for the
suboptimal previous performance of RF in the $\exp-\chi^2$ kernel is
the inaccuracy in the approximation of the $\chi^2$ metric.
A significant contribution of this paper is a new analytic series to
approximate the $\chi^2$ kernel. The new series is derived using only elementary techniques and enjoys geometric convergence rate.
Therefore, it is orders of magnitudes better in terms of approximation error than previously proposed approaches~\cite{Vedaldi2012,Li2012}.
Experiments show that this better approximation quality directly translate to better
classification accuracy by using it in conjunction with the RF method to approximate the $\exp-\chi^2$
kernel.

We have also developed another analytical approximation using techniques from Chebyshev polynomials.
However, that other approximation has a slower linear convergence rate due to the use of the Fourier
transform of a non-differentiable function. This Chebyshev approximation and its derivations are also
listed in the paper for the record.

Another research
question we pursue is whether we can also improve the empirical
performance of RF by applying PCA on the generated features. By applying PCA, the
theoretical convergence rate of RF is no longer confined by the Monte Carlo rate
$O(1/\sqrt{d})$, where $d$ is the number of dimensions used in the approximation.
Rather, it becomes dependent on eigenvalues, and with a fast enough eigenvalue
decay rate, the convergence rate can reach $O(1/d)$ or better~\cite{Bartlett2005}
, raising it to the at least the same level as the Nystr{\"o}m approach. The question is then
whether applying PCA on RF would translate to a comparable (or better) empirical performance.



For this question, we exploit out-of-core versions of PCA that add little computational overhead to the RF approximation,
especially when combined with least squares and other quadratic losses, e.g. group LASSO. PCA allows us to reduce the
number of dimensions required for classification and relaxes memory constraints when multiple kernels have to be approximated
by RF. We also explore the use of unlabeled (test) data in order to better estimate the covariance matrix in PCA.
This turns out to improve the performance by better selecting effective frequency components.

The paper is organized as follows: Section 2 summarizes related work. Section 3 describes the $\chi^2$ kernel, where we elaborate the connection between the $\exp-\chi^2$ kernel and the $\chi^2$ test. In Section 4, we
present the new analytical approximation with geometric convergence rate. Section 5 describes the Chebyshev approximation. Section 6 elaborates the out-of-core PCA, Section 7 presents experiment results on PASCAL VOC 2010 and ImageNet ILSVRC 2010 data, and Section 8 concludes the paper.

\section{Related Work}
To our knowledge, the use of the $\chi^2$ kernel for histogram comparison can be traced back to at least 1996~\cite{Schiele1996}. \cite{Chapelle1999} constructed the $\exp-\chi^2$ kernel and used it in SVM-based image classification. They hypothesized that exponential $\chi^2$ is a Mercer kernel, but the real proof was not available until 2004 in the appendix of~\cite{Fowlkes2004}. The $\chi^2$ kernel and $\exp-\chi^2$ has been used in a number of visual classification and object detection systems~\cite{Zhang07,Vedaldi2009,Marszalek2009context,Gonfaus2010} and has been shown to have one of the best performances among histogram kernels~\cite{Bosch2007}. \cite{Pele2010} proposes an extension to the $\chi^2$ kernel that normalizes the $\chi^2$ cross different bins.
Other metrics for histogram comparison include histogram intersection, where an efficient speed-up for testing has also been proposed~\cite{Maji2013}, Hellinger kernel, earth
mover distance~\cite{Rubner1998} and Jenson-Shannon. See~\cite{Vedaldi2012} for a summary and comparisons.

Random Fourier features were proposed by~\cite{Rahimi2007} on translation-invariant kernels. \cite{Sreekanth2010} generalizes it to the $\exp-\chi^2$ kernel by the
aforementioned two-steps approach. Several
other studies on linear kernel approximations also used ideas in RF~\cite{Vedaldi2012,li_dagm10,Li2012}.

The Nystr\"{o}m method \cite{Williams2001} sub-samples the training set and operate on a reduced kernel matrix.
Its asymptotic convergence rate had long known to be slow~\cite{Drineas2005}, but recent papers have proved that
it is actually faster than the Monte Carlo rate of RF~\cite{Jin2012}.
Other speed-ups to kernel methods based on low-rank approximations of the kernel matrix have been proposed in \cite{Bach2005,Fine2001}.

A topic of recent interest is methods for coding image features, where the goal is to achieve good performance using linear learners following a feature embedding \cite{Lee2007,Wang2010LLC}. Hierarchical coding schemes based on deep structures have also been proposed \cite{Lee2009}. Both sparse and dense coding schemes have proven successful, with supervector coding \cite{Lin2011} and the Fisher kernels \cite{Perronnin2010} some of the best performers in the ImageNet large-scale image classification challenge \cite{ImageNet}. The dictionaries of some influential coding schemes are usually extremely large -- both the Fisher kernel and supervector coding usually require more than 200k dimensions \cite{Chatfield2011}) and the training of the dictionary is often time-consuming.
RF and Nystr\"{o}m do not require training, hence they are interesting alternatives to these methods.

A crucial component in many coding algorithms is a max-pooling approach, which uses the maximum of the coded descriptors in a spatial range as features. Since in this case an informative small patch could have the same descriptor as the whole image, it is desirable in image classification (for highlighting important regions) but undesirable for object detection and semantic segmentation problems, where the size and shape of the object is of interest. A recent second-order pooling scheme~\cite{CarreiraO2P} proposes an alternative and has shown successful results in the semantic segmentation problem. These pooling schemes are orthogonal to the feature approximation problem in RF and could be potentially used in conjunction.

\section{The $\chi^2$ kernel and its Relationship with the $\chi^2$ Test}
Throughout the paper we denote $\circ$ element-wise products of vectors, $\bfI$ the identity matrix, $\mathbf{0}$ a column vector of zeros, and $\mathbf{1}$ a column vector of ones. $\frac{\mathbf{a}}{\mathbf{b}}$ denotes an element-wise division of $\mathbf{b}$ from $\mathbf{a}$.

The $\chi^2$ kernel is derived from Pearson's $\chi^2$ test. The original Pearson $\chi^2$ test is for testing whether an empirical histogram estimate matches a probability distribution. Given a histogram estimate $\bfx = [x_1, x_2, \ldots, x_d]$, the test statistic is
\begin{equation}
X^2(\bfx,\mathbf{E}) = \sum_{i=1}^d \frac{(x_i - E_i)^2}{E_i}
\label{eqn:pearson_chi2}
\end{equation}
where $\mathbf{E} = [E_1, E_2, \ldots, E_d]$ is the theoretical frequency in the bins.

Suppose we have two histogram estimates $\bfx$ and $\bfy$, one can
arrive at a symmetric version by taking the harmonic mean $H(x,y) = \frac{2}{1/x + 1/y}$ of each
bin and sum it up:
\begin{eqnarray}
\chi^2(\bfx,\bfy) &=& \frac{1}{4} \sum_{i=1}^d H\left(\frac{(x_i - y_i)^2}{y_i},\frac{(y_i - x_i)^2}{x_i}\right) \nonumber \\
&=& \frac{1}{2}\sum_{i=1}^d \frac{(x_i - y_i)^2}{x_i + y_i}
\label{eqn:chi2_harmonic}
\end{eqnarray}
The virtue of such a harmonic mean approach lies in removing the singular points in the kernel: the value of the original $\chi^2$ test goes to infinity when $E_i = 0$ and $x_i \neq 0$. Using the harmonic mean approach in (\ref{eqn:chi2_harmonic}), the function is well-defined in all cases.

In order to use the original $\chi^2$ test (\ref{eqn:pearson_chi2}) to determine goodness of the fit, one needs to compute the p-value of the $\chi^2$ statistic:
\begin{equation}
p = 1 - P(\frac{k}{2}, \frac{X^2}{2})
\end{equation}
where $k$ is the degree of freedom in the distribution, $P(k,x)$ is the regularized Gamma function. The p-value is $1$ minus the cumulative distribution function (CDF) the test statistic. If a p-value is small, then it means the observed statistic is very unlikely to happen under the hypothesized distribution. A usual criterion is to decide that $\bfx$ disagrees from the distribution specified by $[E_1, \ldots, E_d]$ if $p < 0.05$. In the case of the $\chi^2$ test, with a special case of $k=2$, one has $p = \exp(-\frac{X^2}{2})$.

As an analogy one can define the $\exp-\chi^2$ kernel based on the harmonic $\chi^2$ kernel:
\begin{equation}
k(\bfx,\bfy) = \exp(- \beta (1 - \chi^2(\bfx, \bfy)))
\end{equation}
with $\beta$ being a kernel parameter. Note although such a kernel has been used in many papers and enjoys
excellent results, we have not found an elaboration of the analogy with the $p$-value of the $\chi^2$ test in literature. Since the $p$-value is the relevant metric for comparing two distributions, the $\exp-\chi^2$ kernel can be considered intuitively better than the $\chi^2$ function as a similarity metric comparing two histogram distributions. Empirically, we have tested kernels with different degrees of freedom, and found out that $\exp-\chi^2$ works similarly to $\mathrm{erf}(\chi^2)$ (corresponding to $\chi^2$ with $1$ degree of freedom) while outperforming all others with more than $2$ degrees of freedom.

\section{An Analytical Approximation to the $\chi^2$ kernel with Geometric Convergence}
In the following we show an analytical approximation to the $\chi^2$ kernel, referred to as the direct approximation later in the paper. We start with the one-dimensional case. The $\chi^2$ kernel in one dimension has the form:
\begin{equation}
\chi^2(x_i,y_i) = \frac{(x_i-y_i)^2}{2(x_i+y_i)} = \frac{1}{2}(x_i+y_i) - \frac{2x_iy_i}{x_i+y_i}
\end{equation}
Because $\sum_i x_i = 1$ and $\sum_i y_i = 1$ in a histogram, the first $x_i+y_i$ form sums to a constant. It is thus important to represent the form $\frac{xy}{x+y}$ into an inner product. We will make repeated use of the following crucial formula
\setlength{\arraycolsep}{0.0em}
\begin{eqnarray}
\frac{2xy}{x+y} &=& \frac{x-k}{x+k} \frac{y-k}{y+k} \frac{2xy}{x+y} + \left(1 -\frac{x-k}{x+k} \frac{y-k}{y+k}\right) \frac{2xy}{x+y} \nonumber\\
&=& \frac{x-k}{x+k} \frac{y-k}{y+k} \frac{2xy}{x+y} + \frac{2\sqrt{k}x}{x+k}\frac{2\sqrt{k}y}{y+k}
\label{eqn:crucial_formula}
\end{eqnarray}
\setlength{\arraycolsep}{5pt}
so that $\frac{2\sqrt{k}x}{x+k}\frac{2\sqrt{k}y}{y+k}$ gives a one-term linear approximation of the $\chi^2$ kernel. Repeatedly plugging (\ref{eqn:crucial_formula}) into the $\frac{2xy}{x+y}$ in the first term of the right-hand-side in (\ref{eqn:crucial_formula}) gives us a series with multiple parameters:
\begin{equation}
\mathbf{c}_x = \left[\frac{2\sqrt{k_1}x}{x+k_1}, \frac{x-k_1}{x+k_1}\frac{2\sqrt{k_2}x}{x+k_2}, \frac{x-k_1}{x+k_1}\frac{x-k_2}{x+k_2}\frac{2\sqrt{k_3}x}{x+k_3} \ldots \right]
\label{eqn:cseries}
\end{equation}
This series has geometric convergence rate as the N-term error
is exactly:
\begin{equation}
\frac{(x-k_1)\ldots(x-k_N)(y-k_1)\ldots(y-k_N)}{(x+k_1)\ldots (x+k_N)(y+k_1)\ldots (y+k_N)} \frac{2xy}{x+y}
\end{equation}
which is straightforwardly geometric if we take $k = k_1 = \ldots =k_N$, because $|\frac{x-k}{x+k}| < 1, \forall 0<k\leq 1$.

We see the multiple parameters $k_1, k_2, \ldots$ in this series a boon rather than a distraction, because
a trick involving multiple parameters is needed to achieve excellent convergence rate of the histogram approximation in the full domain of $[0,1]$. Note that the convergence rate is dominated by $\left(\frac{x-k}{x+k}\right)^{2N}$, if there is only one $k = k_1 = \ldots =k_N$. In this case, the convergence rate can be very slow if $\frac{x-k}{x+k}$ is close to $1$. Two examples are: $k=1, x = 0.005$ and $k=0.005, x = 1$, in both cases $\frac{x-k}{x+k} \approx 0.99$ and even geometric convergence is very slow. From the above example one can see that there is no single $k$ choice that achieves good convergence rate on the entire input domain $[0,1]$. Our solution is to utilize multiple different parameters to cover different regions, and combining the parameter choice with the input distribution of our data to achieve an optimal convergence rate on the entire domain of the input.

First we establish a simple upper bound of the function to facilitate simpler error computation:
\begin{equation}
\frac{2xy}{x+y} \leq \frac{2x}{x+1}
\end{equation}

Now the $N$-term error can be bounded as:
\begin{equation}
\frac{2xy}{x+y} - c_x^\top c_y \leq \frac{2(x-k_1) \ldots (x-k_N) x}{(x+k_1) \ldots (x+k_N) (x+1)}
\end{equation}

Our algorithm for finding the parameters proceeds greedily to eliminate the highest error peak at each iteration. Specifically, we choose the parameter:
\begin{equation}
k_{N+1} = \arg \max_x \left|\frac{2(x-k_1) \ldots (x-k_N) x}{(x+k_1) \ldots (x+k_N) (x+1) }p(x)\right|
\end{equation}
where $p(x)$ is the input distribution of $x$, estimated on each particular dataset. Such a choice reduces error to $0$ at the mode of the input distribution and is empirically tested to be superior than other greedy schemes such as minimizing the mean error at each step. In practice, $p(x)$ is estimated using a histogram estimate with logarithmically spaced bins, and $k_{N+1}$ is chosen as one of the bin centers.
The algorithm of such an implementation is shown in Algorithm~\ref{alg:est_param}.

Note that the $\chi^2$ kernel in this form coincides with the harmonic mean of the two vectors. Therefore, our approach could also be a linear approximation on the harmonic mean between two vectors. However, currently we do not know of other applications of that.

\begin{algorithm}[h]
\begin{algorithmic}[1]
\INPUT: Feature matrix $\bfX$, Parameter vector length $N$.
\OUTPUT: parameter vector $\bfk$.
\STATE Compute a histogram density estimate $\bfh$ on all nonzero values in $\bfX$ using logarithmically spaced bins in the range $[\min_{x \in \bfX, x \neq 0} \bfX, \max \bfX]$, denote the vector of bin centroids as $\bfx$.
\STATE $\bfb = \frac{\bfx}{\bfx+1} \circ \bfh$
\FOR{$i=1 \to N$}
\STATE $k_i = x_j, j = \arg \max_j \left|b_j \right|$
\STATE $\bfb = \bfb \circ \frac{\bfx - k_i}{\bfx + k_i}$
\ENDFOR
\end{algorithmic}
\caption{Find the parameters for the input distribution specified by feature matrix $\bfX$.}
\label{alg:est_param}
\end{algorithm}

Given the approximated $\chi^2$ kernel, we follow \cite{Sreekanth2010} to apply standard RF on a Gaussian
 kernel~\cite{Rahimi2007} over the approximation from (\ref{eqn:cseries}) in order
 to obtain an $\exp-\chi^2$ kernel. The entire algorithm is shown in Algorithm~\ref{alg:cheb}.

\begin{algorithm}
\begin{algorithmic}[1]
\INPUT: $n\times d$ data matrix $\bfX = [\bfX_1^T, \bfX_2^T, \ldots, \bfX_n^T]^T$. Parameters $N, D$.
\OUTPUT: The random Fourier feature $\bfZ$ of the exp-$\chi^2$ kernel.
\STATE Compute parameter vector $k$ using $\bfX,N$ and Algorithm~\ref{alg:est_param}.
\STATE Compute for $q=1, \ldots, N$
\begin{equation*}
c_{q}(x_{ij}) = \left(\prod_{p=1}^{q-1} \frac{x_{ij}-k_p}{x_{ij}+k_p}\right)\frac{2 \sqrt{q} x_{ij}}{x+k_q}
\end{equation*}
where $x_{ij}$ represents dimension $j$ of the example $\bfX_i$. Denote $\mathbf{C}(\bfX_i)$ the $Nd \times 1$ vector
constructed by concatenating
 all $c_q(x_{ij}), j = 1, \ldots, d$.
\STATE Construct a $md \times D$ matrix $\mathbf{\Omega}$, where each entry is sampled from a normal distribution $\mathcal{N}(0, 2\gamma)$.
\STATE Construct a $D \times 1$ vector $\bfb$ which is sampled randomly from $[0, 2\pi]^D$.
\STATE $\bfZ_i = \cos(\mathbf{C}(\bfX_i) \mathbf{\Omega}+\bfb)$ is the RF feature for $\bfX_i$ \cite{Rahimi2007}.
\caption{Approximation of the exp-$\chi^2$ kernel based on the direct approximation of the $\chi^2$ distance.}\label{alg:cheb}
\end{algorithmic}
\end{algorithm}
\section{The Chebyshev Approximation}
Denoting $\Delta = \log y
- \log x$, in each dimension of the $1 - \chi^2$ kernel we have
{\small
\begin{equation}
k_0(x,y) = \frac{2xy}{x+y} = \sqrt{xy} \frac{2}{\sqrt{\frac{x}{y}} + \sqrt{\frac{y}{x}}} = \sqrt{xy} \mathrm{sech}(\frac{\Delta}{2}),
\end{equation}
}
where $\mathrm{sech}(x) = \frac{2}{e^x + e^{-x}}$ is the hyperbolic secant function whose  Fourier transform is
$\pi \mathrm{sech}(\pi \omega)$. Using the inverse Fourier transform to map $\pi \mathrm{sech}(\pi \omega)$ back to $k_0(x,y)$
{\small
\begin{eqnarray}
k_0(x,y)&=&\sqrt{xy} \int_{-\infty}^\infty e^{j \omega (\log x - \log y)} \mathrm{sech}(\pi \omega) d \omega \nonumber\\
&=&\int_{-\infty}^\infty \Phi_\omega(x)^* \Phi_\omega(y) d \omega
\end{eqnarray}
}
where $\Phi_\omega(x) = \sqrt{x} e^{-j \omega \log x} \sqrt{\mathrm{sech}(\pi \omega)}$.

Because the kernel is symmetric, the imaginary part of its inverse Fourier transform is $0$, leading to
{\small
\begin{eqnarray}
k_0(x,y) & = & \sqrt{xy} \int_{-\infty}^\infty \cos(\omega (\log x - \log y)) \mathrm{sech}(\pi \omega) d \omega \nonumber \\
& = & \sqrt{xy} \int_{-\infty}^\infty (\cos(\omega \log x) \cos(\omega \log y) \\
&& + \sin(\omega \log x) \sin(\omega \log y))\frac{2}{e^{\pi \omega} + e^{-\pi \omega}} d \omega. \nonumber
\end{eqnarray}
}
Through a change of variable, $z = 2 \arctan e^{\pi \omega}$, the integral becomes
{\small
\begin{flalign}
&k_0(x,y)= \\
&\frac{\sqrt{xy}}{\pi} \int_0^{\pi} (\cos(\frac{1}{\pi} \log|\tan \frac{z}{2}| \log x) \cos(\frac{1}{\pi} \log|\tan \frac{z}{2}| \log y) \nonumber \\
&+ \sin(\frac{1}{\pi} \log |\tan \frac{z}{2}| \log x) \sin(\frac{1}{\pi} \log |\tan \frac{z}{2}| \log y)) dz. \nonumber
\end{flalign}
}
Since the functions $\cos(\frac{1}{\pi} \log|\tan \frac{z}{2}| \log x)$ and $\sin(\frac{1}{\pi} \log|\tan \frac{z}{2}| \log x)$ are periodic and even, they can be represented using discrete-term Fourier cosine series
{\small
\begin{equation}
f_x(z) = \frac{a_0(x)}{2} + \sum_{n=1}^N a_n(x) \cos(nz).
\end{equation}
}
Since for all integers $n$ and $m$,
{\small
\begin{align*}
\int_0^\pi \cos(nx) \cos(mx) dx = \begin{cases} 0& n \neq m \\ \pi/2 & n=m\end{cases},
\end{align*}
}
we have
{\small
\begin{equation}
\frac{1}{\pi} \int_0^\pi f_x(z) f_y(z) dz = \frac{a_0(x) a_0(y)}{4} + \frac{1}{2} \sum_i a_i(x) a_i(y)
\end{equation}
}
which offers a natural orthogonal decomposition. A vector $a_x = \frac{1}{\sqrt{2}}[a_0(x)/ \sqrt{2}, a_1(x), a_2(x), \ldots, a_n(x)]$ guarantees that $a_x^T a_y = \frac{1}{\pi} \int_0^\pi f_x(z) f_y(z) dz$.

Now, to determine the coefficients which are
{\small
\begin{eqnarray}
a_q(x) & = & \frac{2}{\pi} \int_0^\pi \cos(\frac{1}{\pi} \log \tan(\frac{z}{2}) \log x) \cos(qz) dz \nonumber \\
b_q(x) & = & \frac{2}{\pi} \int_0^\pi \sin(\frac{1}{\pi} \log \tan(\frac{z}{2}) \log x) \cos(qz) dz
\end{eqnarray}
}
we use integration-by-parts to derive an analytical recurrence relation
(See Appendix):
\begin{IEEEeqnarray}{rCl}
b_q(x)&=&\frac{\pi}{\log x}\left(\frac{q+1}{2} a_{q+1}(x) - \frac{q+1}{2}a_{q-1}(x)\right) \nonumber \\
a_q(x)&=&- \frac{\pi}{\log x} \left(\frac{q+1}{2} b_{q+1}(x) - \frac{q-1}{2} b_{q-1}(x)\right), q > 0 \nonumber \\
a_0(x)&=&- \frac{\pi}{\log x} b_{1}(x)
\label{eqn:recur}
\end{IEEEeqnarray}

Now we can combine the nonzero entries for the two series and write it as $d_q$, and the recurrence relation can also be written out for $d_q$ as:
{\small
\begin{equation}
\label{eqn:cseries}
d_{q}(x) = \left\{\begin{array}{ll} \frac{1}{q} ((-1)^{q} \frac{2\log x}{\pi} d_{q-1}(x) + (q-2) d_{q-2}(x)), & q> 1 \\
- \frac{\sqrt{2} \log x}{\pi} d_0(x), & q = 1 \\
\frac{2x}{x+1}, & q = 0 \end{array} \right.
\end{equation}
}
with $k_0(x,y) = \sum_q d_q(x) d_q(y)$. This approximation can be used instead of the series $c_q$ in Algorithm~\ref{alg:cheb}.

We refer to the above approximation as the Chebyshev approximation because it draws ideas from Chebyshev polynomials and the Clenshaw-Curtis quadrature \cite{ChebyshevBook}. A central idea in the Clenshaw-Curtis quadrature is to use the change of variable $\theta = \arccos(x)$ in order to convert an aperiodic integral into a periodic one, making possible to apply Fourier techniques. Our variable substitution $z = \arctan e^x$ serves a similar purpose. The same technique can be applied in principle to other kernels, such as the histogram intersection and the Jensen-Shannon kernel. However, the integration by parts used to derive the analytical approximation may not extend straightforwardly.

\subsection{Convergence Rate of the Chebyshev Approximation}
In this section we present an simple analysis on the asymptotic convergence rate of the Chebyshev approximation.
Since (\ref{eqn:cseries}) is exact, we can apply standard results on Fourier series coefficients \cite{ChebyshevBook}, which state the convergence rate depends on the smoothness of the function that is approximated.
\begin{lemma}
$|k_0(x_i,y_i) - \sum_{q=1}^N d_q(x_i) d_q(y_i)| \leq \frac{C}{N} \sqrt{x_i y_i}$ where $C$ is a constant.
\label{lem:converge}
\end{lemma}
\begin{proof}
Since $\frac{d_N(x_i)}{\sqrt{x_i}}$ represents Fourier series for $\cos(\frac{1}{\pi}\log|\tan\frac{z}{2}| \log x_i)$ and $\sin(\frac{1}{\pi}\log|\tan\frac{z}{2}| \log x_i)$, which are both absolutely continuous but not continuously differentiable (oscillate at $z=0$), we have:
{\small
\begin{equation}
0 < N d_N(x_i) \leq \sqrt{C} \sqrt{x_i}
\end{equation}
}
and consequently
{\small
\begin{equation*}
|k_0(x_i,y_i) - \mathbf{d}(x_i)^T \mathbf{d}(y_i)| \leq \sum_{k>N} \frac{C}{N^2} \sqrt{x_i y_i} \leq \frac{C}{N} \sqrt{x_i y_i}  \qedhere
\end{equation*}}
\end{proof}
Using Lemma \ref{lem:converge} it is straightforward to prove that
\begin{theorem}
$|k_0(x,y) - \sum_i \sum_{k=1}^N d_q(x_i) d_q(y_i)| \leq \frac{C}{N}$ when $\sum_i x_i = \sum_i y_i = 1$.
\end{theorem}
\begin{proof}
We use Cauchy-Schwarz inequality, \\$|k_0(x,y) - \sum_i \sum_{q=1}^N d_q(x_i) d_q(y_i)| \leq \frac{C}{N} \sum_i \sqrt{x_i y_i}$ \\ $\leq \frac{C}{N} \sqrt{\sum_i x_i \sum_i y_i} = \frac{C}{N}$.
\end{proof}

Although this approximation is also analytic, it converges slower than the $c_q$ series in (\ref{eqn:cseries}). In the experiments it is also shown that it has inferior results than the direct approximation. However, the convergence of two different analytical series to the same function may lead to further mathematical equalities, therefore we still listed the Chebyshev approximation in the paper.

\section{Principal Component Analysis of Random Features on Multiple Descriptors}
Another rather orthogonal strategy we pursue is principal component analysis after obtaining random features, and solving regression problems after the PCA. Care needs to be exercised when PCA is performed on an extremely large-scale dataset in conjunction with multiple kernels. Similar approaches have been discussed extensively in the high-performance computing literature ((e.g., \cite{Qu2002}).

The main advantage of using PCA after RF (hereafter called RF-PCA) is to reduce the memory footprint. It is known that the performance of RF improves when more random dimensions are used. However,
the speed of learning algorithms usually deteriorates quickly when the data cannot be load in memory, which would be the case when the RF of multiple kernels are concatenated: e.g. with 7 kernels and 7,000 RF dimensions for each kernel, the learning phase following RF needs to operate on a 49,000 dimensional feature vector.

Using eigenvectors is also one of the very few approaches that could provide a better asymptotic convergence rate than the $O(\frac{1}{\sqrt{m}})$ for Monte Carlo, which in this case means to use fewer approximation dimensions for the same quality. Many other techniques like quasi-Monte Carlo suffer from the curse of dimensionality -- the convergence rate decreases exponentially with the number of input dimensions \cite{Caflisch1998}, which generally makes it unsuitable for RF which is supposed to
work on high-dimensional problems.

Another interesting aspect of RF-PCA is it can bring an unexpected flavor of semi-supervised learning, in that one can use unlabeled test data to improve classification accuracy. RF-PCA amounts to selecting the relevant dimensions in the frequency domain, by considering both the training and testing data during PCA, frequencies that help discriminate test data will more likely be selected. In the experiments such a strategy will be shown to improve performance over the computation of PCA only on training data.

One main problem is, in a large training set, the feature matrix cannot be fully loaded into memory. Therefore PCA needs to be performed out-of-core, a high-performance computing term depicting this situation (unable to load data into memory). The way to do PCA in linear time is not by singular value decomposition on the RF features $\bfZ$, but rather by performing eigenvalue decomposition for the centered covariance matrix $\bfZ^T (\bfI - \frac{1}{n} \mathbf{11}^T) \bfZ$. $\bfZ^T \bfZ = \sum_i \bfZ_i^T \bfZ_i$ can be computed out-of-core by just loading a chunk of $\bfX_i$ into memory at a time, compute their RF feature $\bfZ$, compute the covariance matrix and then delete the RF features from memory. Then an eigen-decomposition gives the transformation matrix $\bfU$ for PCA. We denote $\bar{\bfU}$ as the matrix obtained by selecting the first $D$ dimensions of $\bfU$ corresponding to the largest eigenvalues. Denote the mean vector of the input matrix $\bar{\bfZ} = \frac{1}{n} \bfZ^T \mathbf{1}$, and
{\small
\begin{equation}
\tilde{\bfZ} = (\bfZ-\mathbf{1}\bar{\bfZ}^T) \bar{\bfU} = (\bfI - \frac{1}{n} \mathbf{11}^T) \bfZ \bar{\bfU}
\end{equation}
}
is the feature vector obtained after PCA projection (Algorithm \ref{alg:simplepca}).
\begin{algorithm}
\begin{algorithmic}[1]
\INPUT: $n\times d$ data matrix $\bfX = [\bfX_1^T, \bfX_2^T, \ldots, \bfX_n^T]^T$. Output vector $y$. Number of dimension $D$ to retain after PCA.
\STATE Divide the data into $k$ chunks, called $\bfX_{(1)},$ $\bfX_{(2)}, \ldots,$ $\bfX_{(k)}.$
\STATE $\bfH = \mathbf{0}, \bfm = \mathbf{0}, \bfv = \mathbf{0}$
\FOR{$i=1\to k$}
\STATE Load the $i$-th chunk $\bfX_{(i)}$ into memory.
\STATE Use Algorithm \ref{alg:cheb} to compute the RF feature $\bfZ_{(i)}$ for $\bfX_{(i)}$.
\STATE $\bfH = \bfH + \bfZ_{(i)}^T \bfZ_{(i)}$, $\bfm = \bfm + \bfZ_{(i)}^T 1$, $\bfv = \bfv + \bfZ_{(i)}^T \bfy$
\ENDFOR
\STATE $\bfH = \bfH - \frac{1}{n} \bfm\bfm^T$.
\STATE Compute eigen-decomposition $\bfH = \bfU\mathbf{D}\bfU^T$. Output the first $D$ columns of $\bfU$ as $\bar{\bfU}$, the
diagonal matrix $\mathbf{D}$, and the input-output product $\bfv$.
\caption{Out-of-Core Principal Component Analysis.}\label{alg:simplepca}
\end{algorithmic}
\end{algorithm}
It is very convenient to perform regression with a quadratic loss after PCA, since only the Hessian is needed for optimization. This applies not only to traditional least squares regression, but also to the LASSO, group LASSO, and other composite regularization approaches. In this case the projections need not be performed explicitly. Instead, notice that only $\tilde{\bfZ}^T \tilde{\bfZ}$ and $\tilde{\bfZ}^T y$ are needed for regression:
{\small
\begin{eqnarray}
\tilde{\bfZ}^T \tilde{\bfZ} & = & \bar{\bfU}^T \bfZ^T (\bfI-\frac{1}{n} \mathbf{11}^T) \bfZ \bar{\bfU} \nonumber \\
\tilde{\bfZ}^T \bfy & = & \bar{\bfU}^T \bfZ^T (\bfI-\frac{1}{n} \mathbf{11}^T) \bfy
\end{eqnarray}
}
It follows that only $\bfZ^T \bfZ$, $\bfZ^T \mathbf{1}$ and $\bfZ^T \bfy$ have to be computed. All terms can be computed out-of-core simultaneously. Algorithm \ref{alg:simplereg} depicts this scenario.
\begin{algorithm}[h]
\begin{algorithmic}[1]
\INPUT: $n\times d$ data matrix $\bfX = [\bfX_1^T, \bfX_2^T, \ldots, \bfX_n^T]^T$. Output vector $\bfy$. Number of dimension $D$ to retain after PCA.
\STATE Perform out-of-core PCA using Algorithm \ref{alg:simplepca}.
\STATE $\bfH' = \bar{\bfU}^T \bfH \bar{\bfU} = \bar{\mathbf{D}}$, the first $D$ rows and columns of the diagonal matrix $\mathbf{D}$.
\STATE $v' = \bar{\bfU}^T \bfv - \frac{1}{n} (\mathbf{1}^T \bfy) \bar{\bfU}^T \bfm$.
\STATE Perform learning on $\bfH', \bfv'$, e.g., for linear ridge regression where the optimization is $\arg \min_\bfw \|\bfw^T \tilde{\bfZ} - \bfy\|^2 + \lambda \|\bfw\|^2$, the solution is $\bfw = (\bfH'+\lambda \bfI)^{-1} \bfv'$.
\STATE Use $\bar{\bfU}^T \bfw$ instead of $\bfw$ as a function of the original inputs: $f(x) = \bfw^T \bar{\bfU} \bfx - \frac{1}{n} \bfw^T \bar{\bfU} \bfm$, in order to avoid the projection for the testing examples.
\caption{Learning after PCA with Quadratic Loss.}\label{alg:simplereg}
\vskip -0.05in
\end{algorithmic}
\end{algorithm}
Under this PCA approach the data is loaded only once to compute the Hessian. Additional complexity of $O(D^3)$ is necessary for matrix decomposition on $\bfH$. If ridge regression is used, the $\bfH'$ after decomposition is diagonal therefore only $O(D)$ is needed to obtain the regression results. In this case the additional constant factor is quite small. The bottleneck of this algorithm for large-scale problems is undoubtedly the computation of the initial Hessian, which involves reading multiple chunks from disk.

The more sophisticated case is when PCA needs to be performed separately on multiple different kernel approximators, i.e., $\bfZ = [\bfZ^{(1)} \bfZ^{(2)} \ldots \bfZ^{(l)}]$, where each $\bfZ^{(i)}$ is the RF feature embedding of each kernel. This time, the need to compute $\bfZ^{(i)^T} \bfZ^{(j)}$ rules out tricks for simple computation. The data needs to be read in twice (Algorithm \ref{alg:hardpca}), first to perform the PCA, and then use $\bfU$ to transform $\bfX$ in chunks in order to obtain $\bfZ$ and $\bfZ^T \bfZ$. But the full computation is still linear in the number of training examples.
In both cases, the projection is not required for the testing examples. Because whenever $\bfw$ is obtained, $\bfw^T \tilde{\bfZ} =  \bfw^T \bar{\bfU} (\bfZ - \frac{1}{n}\bar{\bfZ}\mathbf{1}^T)$, then $\bar{\bfU}\bfw$ can be the weight vector for the original input, with the addition of a constant term.
\begin{algorithm}[h]
\begin{algorithmic}[1]
\INPUT: $n\times d$ data matrix $\bfX = [\bfX_1^T, \bfX_2^T, \ldots, \bfX_n^T]^T$. Output vector $\bfy$. Number of dimension $D$ to retain after PCA.
\STATE Perform out-of-core PCA using Algorithm \ref{alg:simplepca}.
\FOR{$i=1\to k$}
\STATE Load the $i$-th chunk $\bfX_{(i)}$ into memory.
\STATE Use Algorithm 1 to compute the RF feature $\bfZ_{(i)}$ for $\bfX_{(i)}$, with the same randomization matrix $\mathbf{\Omega}$ as before.
\STATE $\tilde{\bfZ} = (\bfZ_{(i)} - \frac{1}{n} \mathbf{1} \bfm^T)\bar{\bfU}$.
\STATE $\bfH' = \bfH' + \tilde{\bfZ}^T \tilde{\bfZ}$, $\bfv' = \bfv' + \tilde{\bfZ}^T \bfy$
\ENDFOR
\STATE Perform learning on $\bfH', \bfv'$. E.g., for linear least squares where the optimization is $\arg \min_\bfw \|\bfw^T \bfZ - \bfy\|^2$, the solution is $\bfw = \bfH'^{-1} \bfv'$.
\STATE Use $\bar{\bfU}^T \bfw$ instead of $w$ as a function of the original inputs: $f(x) = \bfw^T \bar{\bfU} \bfx - \frac{1}{n} \bfw^T \bar{\bfU} \bfm$, in order to avoid the projection step for the testing examples.
\caption{Two-stage Principal Component Analysis when learning with multiple kernels.}\label{alg:hardpca}
\end{algorithmic}
\end{algorithm}
\vskip -0.05in
It is worth noting that out-of-core least squares or ridge regression scales extremely well with the number of output dimensions $c$, which can be used to solve one-against-all classification problems with $c$ classes. In the out-of-core case, $\bfZ^T \bfy$ will be computed in $O(nDc)$ time along with the Hessian in Algorithm \ref{alg:simplepca} or \ref{alg:hardpca}. After the inverse of Hessian is obtained, only a matrix-vector multiplication costing $O(D^2 c)$ is needed to obtain all the solutions, without any dependency on $n$. Thus the total time of this approach with $c$ classes is $O(nDc + D^2 c)$ which scales very nicely with $c$. Especially compared with other algorithms that need to perform the full training procedure on each class. Although the $L_2$ loss is not optimal for classification, in large-scale problems (e.g. ImageNet) with $1,000 - 10,000$ classes, the out-of-core ridge regression can still be used to generate a fairly good baseline result quickly.

\section{Experiments}
Our experiments are conducted on two challenging datasets: PASCAL VOC 2010 \cite{VOC2010Workshop} and ImageNet \cite{ImageNet} ILSVRC 2010 (http://www.image-net.org/challenges/LSVRC/2010/). These challenging benchmarks reveal the subtle performance differences among approximation methods, which would otherwise be difficult to observe in simple datasets. We conduct most experiments on the medium-scale PASCAL VOC data in order to compare against exact kernel methods. For this dataset, we use exclusively the \texttt{train} and \texttt{val} datasets, which have 964 images and around 2100 objects each. Classification results are also shown on the ImageNet dataset to demonstrate the efficiency of our kernel approximations. The experiments are conducted using an Intel Xeon E5520 2.27GHz with 8 cores and 24GB memory. The algorithm \ref{alg:cheb} is parallelized using OpenMP to take advantage of all cores.
\subsection{Comparing Approximations}
To test different approximations, we consider a small sample from the PASCAL VOC segmentation dataset. For training, we use image segments (obtained using the constrained parametric min-cuts algorithm, CPMC \cite{Carreira2012}) that best match each ground truth segment in terms of overlap (subsequently called best-matching segments) in the \texttt{train} set, plus the ground truth segments. The best-matching segments in the \texttt{val} set are used as test. This creates a medium-scale problem with 5100 training and 964 test segments.

The approximations tested in experiments are \texttt{Chebyshev}, \texttt{VZ} \cite{Vedaldi2012}, \texttt{Direct}.
For reference, we also report classification results for the $\chi^2$ kernel without exponentiating as \texttt{Chi2}, as well as the skewed $\chi^2$ kernel proposed in \cite{li_dagm10} as \texttt{Chi2-Skewed}. Due to the Monte Carlo approximation, different random seeds can lead to quite significant performance variations. Therefore the experiments are all averaged over 50 trials on different random seeds. Within each trial, the same random seeds are used for all methods. For \texttt{PCA-Chebyshev}, the initial sampling is done using three times the final approximating dimensions, and PCA is performed to reduce the dimensionality to the same level as the other two methods. We test the classification performance of these kernels with two different types of features: a bag of SIFT words (BOW) feature of 300 dimensions, and a histogram of gradient (HOG) feature of 1700 dimensions. The classification is done via a linear SVM using the LIBSVM library (empirically we found the LIBLINEAR library produced worse results than LIBSVM in this context with dense features). The $C$ parameter in LIBSVM is validated to 50, the kernel to be approximated is exp-$\chi^2$, with $\beta = 1.5$. For \texttt{VZ}, the period parameter is set to the optimal one specified in \cite{Vedaldi2012}. For each kernel, $5$ dimensions are used to approximate the $\chi^2$ distance in each dimension, which represents a common use case.

\begin{figure}
\begin{tabular}{ll}
\hskip -0.12in \includegraphics[width=120pt]{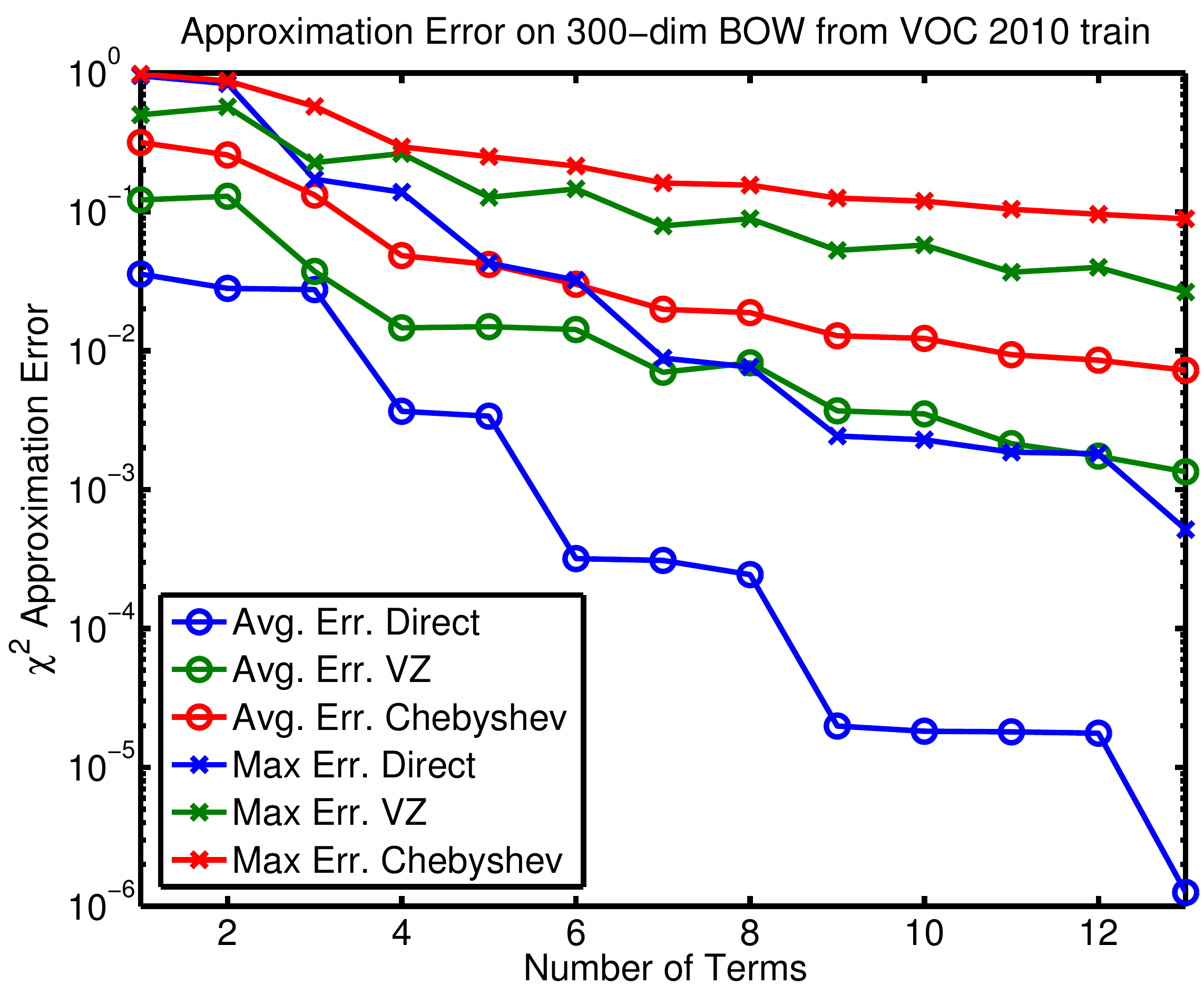} & \hskip -0.12in \includegraphics[width=120pt]{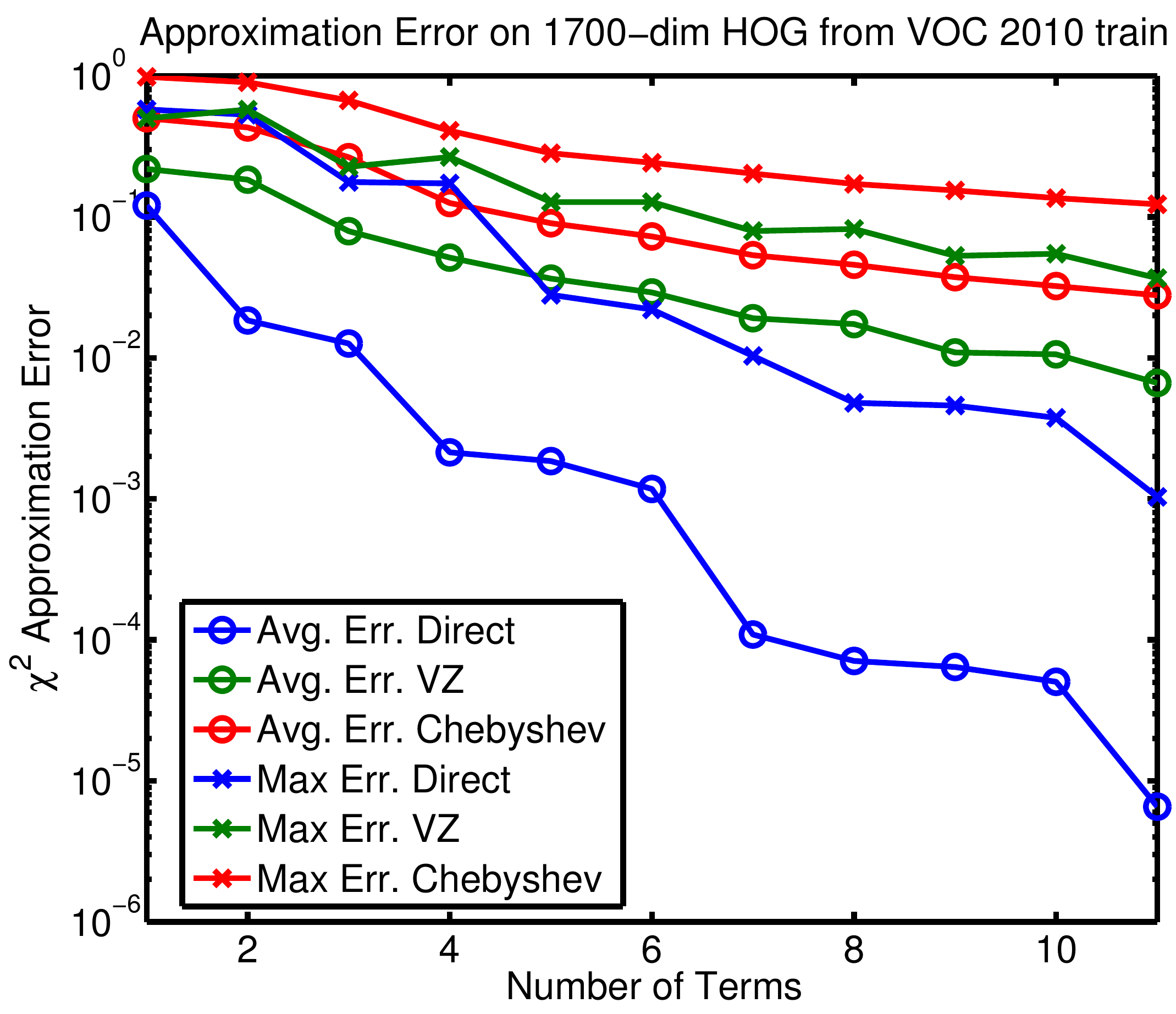}
\end{tabular}
\caption{Comparisons on various approximations to the $\chi^2$ kernel. It can be seen that the
new direct approximation is converging orders of magnitude faster than previous approaches.}
\end{figure}

\begin{figure}
\begin{tabular}{ll}
\includegraphics[width=110pt]{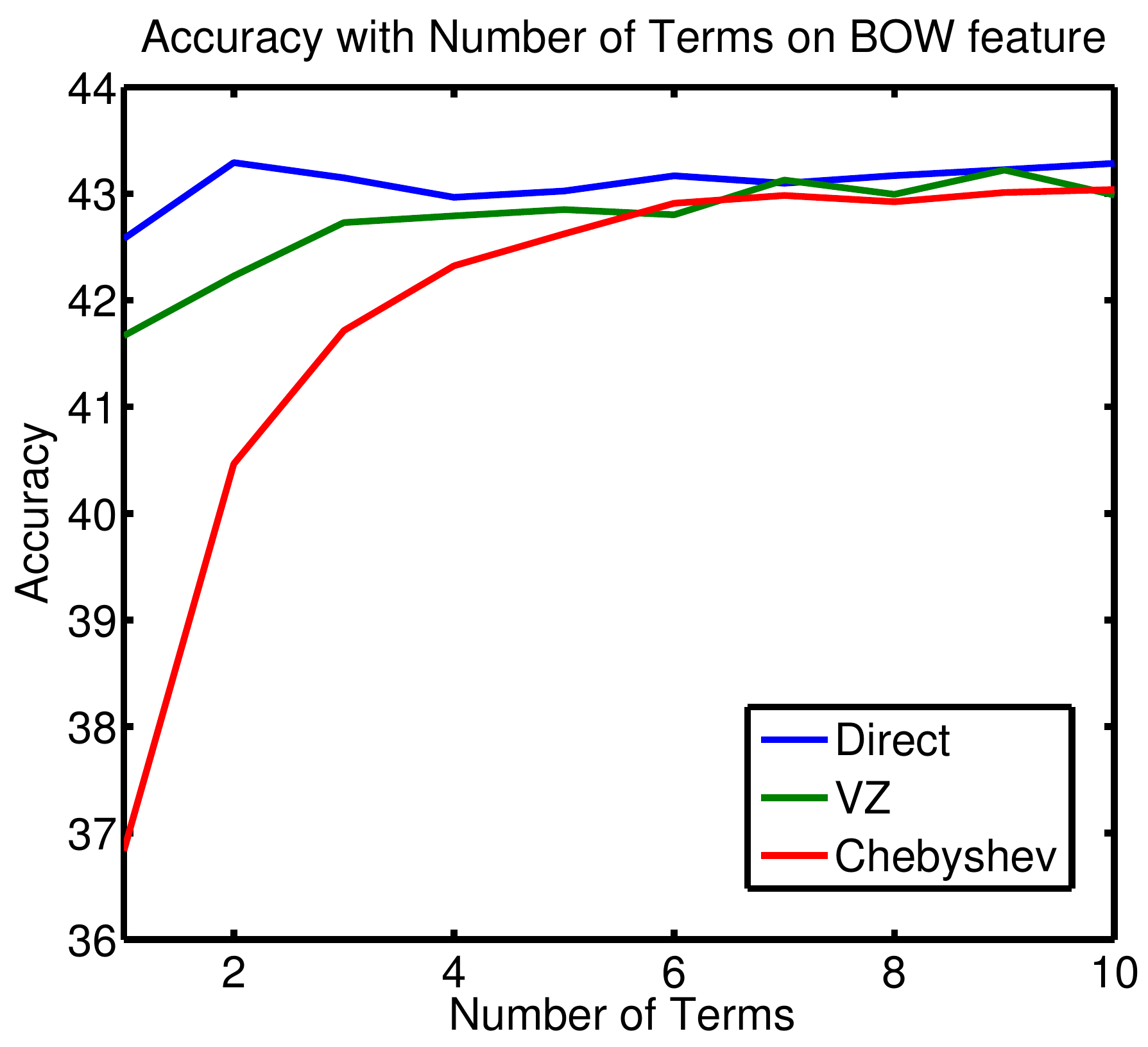} & \includegraphics[width=110pt]{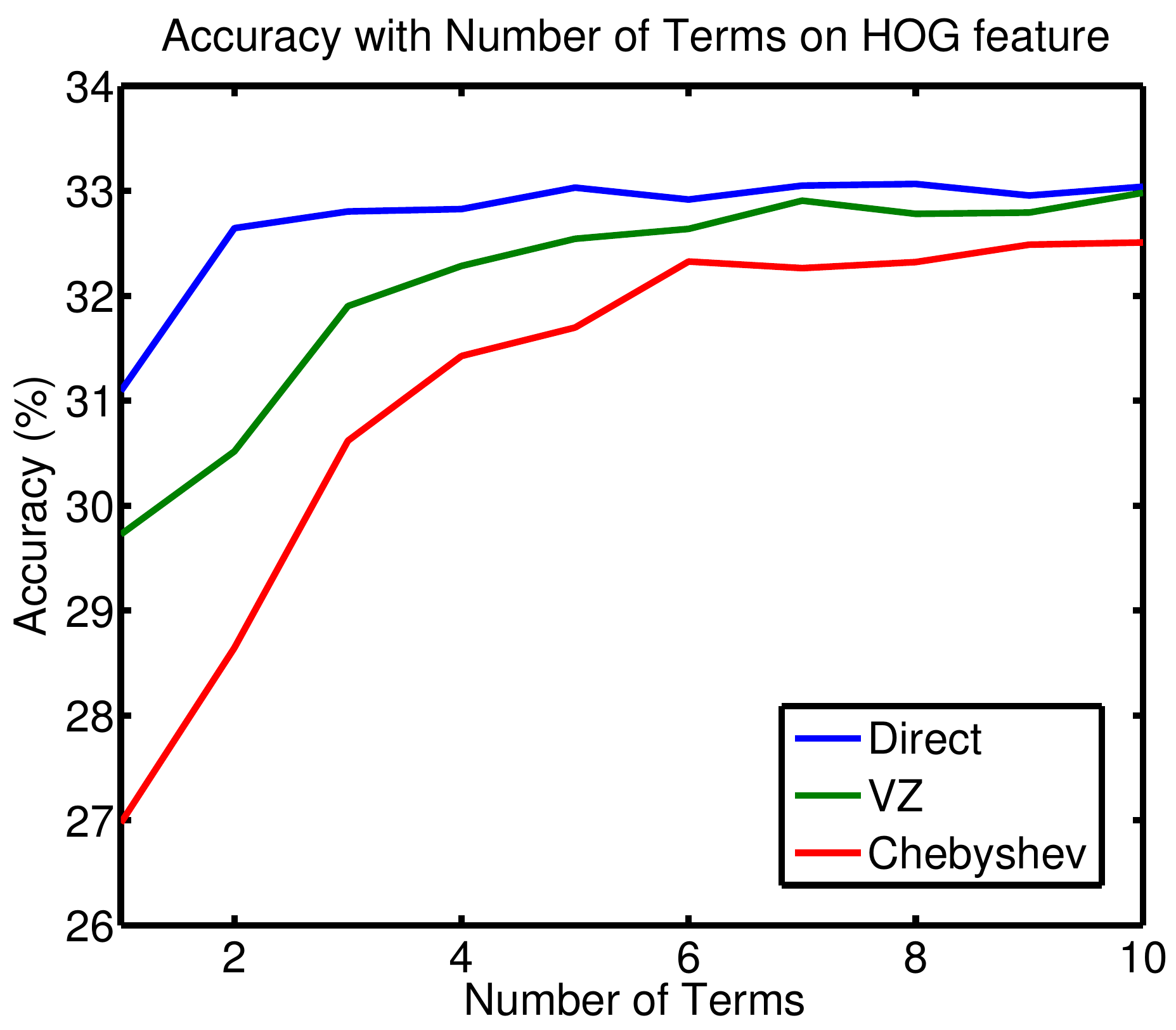}
\end{tabular}
\caption{Effect on the classification accuracy on a 7000-dimensional RF-approximated $\exp-\chi^2$
kernel, using different approximations and various number of dimensions to approximate the $\chi^2$
function. The direct approximation works already very well with 2 terms per original input dimension,
while the other approximation schemes need more terms.}
\end{figure}

\begin{table*}[htb]
\begin{center}
\begin{tabular}{|c|c|c|c|}
\hline
Number of Dimensions & 3000 & 5000 & 7000\\
\hline
Chi2 & 41.91\% & 42.32\% & 42.12\% \\
Chi2-Skewed & 39.82\% $\pm$ 0.73\% & 40.79\% $\pm$ 0.55\% & 40.90\% $\pm$ 0.82\% \\
Chebyshev & $\mathbf{42.03\%} \pm$ 0.80\% & $42.61\% \pm$ 0.65\% & 42.61\% $\pm$ 0.64\% \\
VZ & $\mathbf{42.29\%} \pm$ 0.74\% & $\mathbf{42.82\%} \pm$ 0.63\% & $\mathbf{43.00\%} \pm$ 0.57 \% \\
Direct & $\mathbf{42.09\%} \pm$ 0.79\% & $\mathbf{42.88\%} \pm$ 0.63\% & $\mathbf{43.21}\% \pm$ 0.63\% \\
\hline
\hline
PCA-Chebyshev & $\mathbf{42.80\%} \pm$ 0.74\% & $\mathbf{43.25\%} \pm$ 0.55\% & $\mathbf{43.42\%} \pm$ 0.42 \% \\
PCA-VZ & & & \\
PCA-Direct & $\mathbf{43.16\%} \pm$ 0.55\% & $\mathbf{43.31\%} \pm$ 0.53\% & $\mathbf{43.53\%} \pm$ 0.71 \% \\
\hline
Exact exp-$\chi^2$ & \multicolumn{3}{|c|}{44.19\%} \\
\hline
\end{tabular}
\caption{Classification accuracy of exp-$\chi^2$ kernel when the $\chi^2$ function is estimated with different approximations, on a BOW-SIFT descriptor. Results for the \texttt{Chi2} and \texttt{Chi2-Skewed} kernels are also shown for reference.} \label{tbl:cheb_bow}
\end{center}
\end{table*}

\begin{table*}[htb]
\begin{center}
\begin{tabular}{|c|c|c|c|}
\hline
Number of Dimensions & 3000 & 5000 & 7000\\
\hline
Chi2 & 29.15\% & 30.50\% & 31.22\% \\
Chi2-Skewed & 30.08\% $\pm$ 0.74\% & 30.37 \% $\pm$ 0.63\% & 30.51 \% $\pm$ 0.35 \% \\
Chebyshev & 30.86\% $\pm$ 0.78\% & 31.53\% $\pm$ 0.66\% & 31.90\% $\pm$ 0.70\% \\
VZ & 31.32\% $\pm$ 0.90\% & 32.07 \% $\pm$ 0.83\% & 32.36\% $\pm$ 0.62\% \\
Direct & $\mathbf{31.71\%} \pm$ 0.92\% & $\mathbf{32.72\%} \pm$ 0.73\% & $\mathbf{32.94\%} \pm$ 0.66\% \\
\hline
PCA-Chebyshev & $32.59\% \pm$ 0.77\% & 33.11\% $\pm$ 0.57\% & 33.22\% $\pm$ 0.54\% \\
PCA-VZ & $\mathbf{32.94\%} \pm$ 0.67\% & 33.41\% $\pm$ 0.54\% & $\mathbf{33.45\%} \pm$ 0.59\% \\
PCA-Direct & $\mathbf{32.92\%} \pm$ 0.66\% & $\mathbf{33.68\%} \pm$ 0.57\% & $\mathbf{33.63\%} \pm$ 0.67\% \\
\hline
Exact exp-$\chi^2$ & \multicolumn{3}{|c|}{34.34\%} \\
\hline
\end{tabular}
\caption{Classification accuracy of exp-$\chi^2$ kernel when the $\chi^2$ function is approximated with different approximations, on a HOG descriptor. Results for the \texttt{Chi2} and \texttt{Chi2-Skewed} kernels are also shown for a reference.}
\label{tbl:cheb_hog}
\end{center}
\vskip -0.15in
\end{table*}

\subsection{Results for Multiple Kernels on the PASCAL VOC Segmentation Challenge}
In this section, we consider the semantic segmentation task from PASCAL VOC, where we need to both recognize objects in images, and generate pixel-wise segmentations for these objects. Ground truth segments of objects paired with their category labels are available for training.

A recent state-of-the-art approach trains a scoring function for each class on many putative figure-ground segmentation hypotheses, obtained using CPMC \cite{Carreira2012}. This creates a large-scale learning task even if the original image database has moderate size:
with $100$ segments in each image, training for $964$ images creates a learning problem with around $100,000$ training examples. This input scale is still tractable for exact kernel approaches so that we can directly compare against them.

Two experiments are conducted using multiple kernel approximations for the exp-$\chi^2$ kernels. We use 
7 different image descriptors, which include 3 HOGs at different scales, BOW on SIFT for the foreground and background, and BOW on color SIFT for the foreground and background \cite{Li2010,Carreira2012}. The VOC segmentation measure is used to compare the different approaches. This measure is the average of pixel-wise average precision on the 20 classes plus background. To avoid distraction and for a fair comparison, the post-processing step \cite{Carreira2012} is not performed and the result is obtained by only reporting one segment with the highest score in each image. The method used for nonlinear estimation is one-against-all support vector regression (SVR) as in \cite{Li2010}, and the method for linear estimation is one-against-all ridge regression. The latter is used since fast solutions for linear SVR problems are not yet available for out-of-core dense features. We avoided stochastic gradient methods (e.g., \cite{Lin2011}) since these are difficult to tune to convergence, and such effects can potentially bias the results. We average over 5 trials of different random seeds.

\begin{table}[htb]
\begin{center}
\begin{tabular}{|c|c|}
\hline
Method & Performance \\
\hline
Chebyshev & $26.25\% \pm 0.41\%$ \\
VZ & $25.50\% \pm 0.54\%$ \\
Direct & \\
PCA-Direct & \\
PCA-Chebyshev & $27.57\% \pm 0.44\%$ \\
PCA-training-Chebyshev & $26.95\% \pm 0.35\%$ \\
Nystr{\"o}m & $27.55\% \pm 0.49\% $ \\
\hline
Kernel SVR & $30.47\%$ \\
\hline
\end{tabular}
\caption{VOC Segmentation Performance on the \texttt{val} set, measured by pixel AP with one segment output per image (no post-processing), and averaged over 5 random trials. The upper part of the table shows results on only BOW-SIFT features extracted on foreground and background. The lower part shows results based on combining 7 different descriptors.}
\label{tbl:voc}
\end{center}
\vskip -0.15in
\end{table}

\subsection{Results on ImageNet}
The ImageNet ILSVRC 2010 is a challenging classification dataset where 1 million images have to be separated into
1,000 different categories. Here we only show experiments performed
using the original BOW feature provided by the authors. Our goal is primarily to compare among different approximations,
hence we did not generate multiple image descriptors or a spatial pyramid, which are compatible with our framework
and could improve the results significantly.
Since regression is used, the resulting scores are not well-calibrated across categories. Therefore we perform a
calibration of the output scores to make the 500th highest score of each class the same.

\begin{table*}[htb]
\begin{center}
\begin{tabular}{|c|c|c|c|}
\hline
Number of Dimensions & 3000 & 5000 & 7000\\
\hline
Chebyshev & 16.30\% $\pm$ 0.04\% & 17.11\% $\pm$ 0.04\% & 17.63\% $\pm$ 0.09\% \\
PCA-Chebyshev & $\mathbf{16.66\%} \pm$ 0.08\% & $\mathbf{17.85\%} \pm$ 0.08\% & $\mathbf{18.65\%} \pm$ 0.10 \% \\
VZ & 16.10\% $\pm$ 0.04\% & 16.95 \% $\pm$ 0.08\% & 17.48\% $\pm$ 0.09\% \\
Direct & & & \\
Nystr\"{o}m & & & \\
\hline
Linear & \multicolumn{3}{|c|}{11.6\% (\cite{Deng2010})} \\
\hline
\end{tabular}
\caption{Performance of linear classifiers as well as non-linear approximation methods on ImageNet ILSVRC 2010 data. Notice the significant boost provided by the non-linear approximations (exact non-linear calculations are intractable in this large scale setting).} \label{tbl:imagenet}
\end{center}
\vskip -0.25in
\end{table*}

\section{Conclusion}
The conclusion goes here. The conclusion goes here.The conclusion goes here.The conclusion goes here.The conclusion goes here.The conclusion goes here.The conclusion goes here.The conclusion goes here.The conclusion goes here.The conclusion goes here.The conclusion goes here.The conclusion goes here.The conclusion goes here.The conclusion goes here.The conclusion goes here.The conclusion goes here.The conclusion goes here.The conclusion goes here.The conclusion goes here.The conclusion goes here. The conclusion goes here.The conclusion goes here.The conclusion goes here.The conclusion goes here.

\appendix[Derivation of the Recurrence Relation]
In this appendix we derive the recurrence relations (\ref{eqn:recur}). First we list some useful properties.

\begin{lemma}
Let $u(z) = \frac{1}{\pi} \log \tan(\frac{z}{2}) \log x$, then,
\begin{enumerate}[(1)]

\item $u(\pi) = 0; \cos(u(\frac{\pi}{2})) = 1; \sin(u(\frac{\pi}{2})) = 0$
\item $u'(z) = \frac{\log x}{\pi} \frac{1}{\sin(z)}, \frac{1}{u'(z)} = \frac{ \pi}{\log x} \sin(z)$
\item $\frac{1}{u'(0)} = 0, \frac{1}{u'(\pi)} = 0$
\end{enumerate}
\end{lemma}

First of all we concern $a_0$ and $b_0$. With a symbolic integration software (e.g. Mathematica) one can compute $a_0 = 2 \mathrm{sech}(\frac{\log x}{2}) = \frac{4 \sqrt{x}}{x+1}$ and
$b_0 = 0$. For the rest of the series, we can immediately observe that for $q$ even, $b_q = 0$ because $\sin(u(z))$ is antisymmetric at $\frac{\pi}{2}$ and $\cos(qz)$ is symmetric for even $q$ and antisymmetric for odd $q$. Same argument gets us for $q$ odd, $a_q = 0$. Therefore we only need to solve the coefficients $b_q$
with odd $q$, and $a_q$ with even $q$. Therefore, we start with the integration:
{\small
\begin{IEEEeqnarray*}{cl}
&\frac{\pi}{4} b_k(x)\\
=&\int_0^\frac{\pi}{2} \sin(\frac{1}{\pi} \log \tan(\frac{z}{2}) \log x) \cos(kz) dz \\
=&- \int_0^\frac{\pi}{2} \cos(kz) \frac{1}{u'(z)} d(\cos(u(z))) \\
=&\int_0^\frac{\pi}{2} \cos(u(z)) d(\cos(kz) \frac{\pi}{\log x} \sin(z)) \\
=&\frac{\pi}{\log x} \int_0^\frac{\pi}{2} \cos(u(z)) (-k\sin(kz) \sin(z) + \cos(kz) \cos(z)) dz \\
=&\frac{\pi}{\log x} \int_0^\frac{\pi}{2} \cos(u(z)) (\cos((k+1)z) - (k-1) \sin(kz) \sin(z)) dz \\
=&\frac{\pi}{\log x} \int_0^\frac{\pi}{2} \cos(u(z)) (\cos((k+1)z) \\
&- \frac{k-1}{2} (\cos((k-1)z) - \cos((k+1)z))) dz \\
=& \frac{\pi}{4} \frac{\pi}{\log x} (\frac{k+1}{2} a_{k+1}(x) - \frac{k-1}{2} a_{k-1}(x))
\end{IEEEeqnarray*}
}
where we have used integration-by-parts followed by trigonometric identities. The same trick applies to the $a_k$ series with even coefficients:
{\small
\begin{IEEEeqnarray*}{cl}
&\frac{\pi}{4} a_k(x) \\
= & \int_0^\frac{\pi}{2} \cos(\frac{1}{\pi} \log \tan(\frac{z}{2}) \log x) \cos(kz) dz \\
= & \int_0^\frac{\pi}{2} \cos(kz) \frac{1}{u'(z)} d(\sin(u(z))) \\
= & - \int_0^\frac{\pi}{2} \sin(u(z)) d(\cos(kz) \frac{\pi}{\log x} \sin(z)) \\
= & - \frac{\pi}{\log x} \int_0^\frac{\pi}{2} \sin(u(z)) (-k\sin(kz) \sin(z) + \cos(kz) \cos(z)) dz \\
= & - \frac{\pi}{\log x} \int_0^\frac{\pi}{2} \sin(u(z)) (\cos((k+1)z) - (k-1) \sin(kz) \sin(z)) dz \\
=& - \frac{\pi}{\log x} \int_0^\frac{\pi}{2} \sin(u(z)) (\cos((k+1)z) \\
&- \frac{k-1}{2} (\cos((k-1)z) - \cos((k+1)z))) dz \\
= & - \frac{\pi}{4} \frac{\pi}{\log x} (\frac{k+1}{2} b_{k+1}(x) - \frac{k-1}{2} b_{k-1}(x))
\end{IEEEeqnarray*}
}
For $k=0$ it's slightly different as:
{\small
\begin{IEEEeqnarray*}{rcl}
\frac{\pi}{4} a_0(x) & = & \int_0^\frac{\pi}{2} \cos(u(z)) dz = \int_0^\frac{\pi}{2} \frac{1}{u'(z)} d(\sin(u(z))) \\
& = & - \frac{\pi}{\log x} \int_0^\frac{\pi}{2} \sin(u(z)) \cos(z) dz \\
& = & - \frac{\pi}{4} \frac{\pi}{\log x} b_{1}(x)
\end{IEEEeqnarray*}
}

%

%

\ifCLASSOPTIONcompsoc
  \section*{Acknowledgments}
\else
  \section*{Acknowledgment}
\fi

The authors would like to thank...The authors would like to thank...The authors would like to thank...The authors would like to thank...The authors would like to thank...The authors would like to thank...The authors would like to thank...The authors would like to thank...The authors would like to thank...The authors would like to thank...The authors would like to thank...The authors would like to thank...The authors would like to thank...

\ifCLASSOPTIONcaptionsoff
  \newpage
\fi


\bibliographystyle{IEEEtran}
\bibliography{dependent_reference}

%

\begin{IEEEbiography}{Michael Shell}
Biography text here.
\end{IEEEbiography}

\begin{IEEEbiographynophoto}{John Doe}
Biography text here.Biography text here.Biography text here.Biography text here.Biography text here.Biography text here.Biography text here.Biography text here.Biography text here.Biography text here.Biography text here.Biography text here.Biography text here.Biography text here.Biography text here.Biography text here.Biography text here.Biography text here.Biography text here.Biography text here.Biography text here.Biography text here.Biography text here.Biography text here.Biography text here.Biography text here.Biography text here.Biography text here.Biography text here.Biography text here.Biography text here.Biography text here.
\end{IEEEbiographynophoto}


\begin{IEEEbiographynophoto}{Jane Doe}
Biography text here.Biography text here.Biography text here.Biography text here.Biography text here.Biography text here.Biography text here.Biography text here.Biography text here.Biography text here.Biography text here.Biography text here.Biography text here.Biography text here.Biography text here.Biography text here.Biography text here.Biography text here.Biography text here.Biography text here.Biography text here.Biography text here.Biography text here.Biography text here.Biography text here.Biography text here.Biography text here.Biography text here.
\end{IEEEbiographynophoto}




\end{document}